\documentclass[twoside]{article}

\usepackage[accepted]{arxiv}

\usepackage[square,numbers]{natbib}
\renewcommand{\bibname}{References}
\renewcommand{\bibsection}{\subsubsection*{\bibname}}

\usepackage[utf8]{inputenc} %
\usepackage[T1]{fontenc}    %
\usepackage{url}            %
\usepackage{nicefrac}       %
\usepackage{microtype}      %
\usepackage[normalem]{ulem}

\usepackage{amsfonts}
\usepackage{amsmath}
\usepackage{amssymb}
\usepackage{mathtools}
\usepackage{cancel}

\usepackage{graphicx}
\usepackage[font=small,labelfont=bf]{caption}
\usepackage{subcaption}
\usepackage{adjustbox}
\usepackage{wrapfig}
\graphicspath{{figures/pdf/}}
\DeclareGraphicsExtensions{.pdf}

\usepackage{tikz}
\usetikzlibrary{arrows}
\usetikzlibrary{arrows.meta}
\usetikzlibrary{backgrounds}
\usetikzlibrary{bayesnet}
\usetikzlibrary{calc}
\usetikzlibrary{shapes}

\usepackage{algorithm}
\usepackage{algorithmic}

\usepackage{booktabs}
\usepackage{makecell}
\usepackage{multirow}
\usepackage{rotating}
\usepackage{tabularx}
\usepackage{tabu}

\usepackage{enumitem}
\usepackage{xspace}

\usepackage{thmtools}
\usepackage{thm-restate}

\usepackage{xcolor}
\definecolor{midgreen}{rgb}{0.1,0.5,0.1}
\definecolor{darkgray}{gray}{0.25}
\definecolor{lightblue}{rgb}{0.25,0.25,1}
\usepackage[colorlinks,linkcolor=midgreen,citecolor=darkgray,urlcolor=lightblue]{hyperref}
\usepackage[capitalise]{cleverref}
\creflabelformat{equation}{#2#1#3}

\usepackage{macro}

\begin{document}

\twocolumn[%
    \aistatstitle{On Data Efficiency of Meta-learning}
    \aistatsauthor{Maruan~Al-Shedivat \And Liam~Li \And  Eric~Xing \And Ameet~Talwalkar}
    \aistatsaddress{CMU \And  Determined AI \And CMU, Petuum Inc. \And CMU, Determined AI}%
]

\begin{abstract}%
\hyphenpenalty=1000%
Meta-learning has enabled learning statistical models that can be quickly adapted to new prediction tasks.
Motivated by use-cases in personalized federated learning, we study the often overlooked aspect of the modern meta-learning algorithms---their data efficiency.
To shed more light on which methods are more efficient, we use techniques from algorithmic stability to derive bounds on the transfer risk that have important practical implications, indicating how much supervision is needed and how it must be allocated for each method to attain the desired level of generalization.
Further, we introduce a new simple framework for evaluating meta-learning methods under a limit on the available supervision, conduct an empirical study of \maml, \reptile, and \protonets, and demonstrate the differences in the behavior of these methods on few-shot and federated learning benchmarks.
Finally, we propose \emph{active meta-learning}, which incorporates active data selection into learning-to-learn, leading to better performance of all methods in the limited supervision regime.
\end{abstract}

\section{Introduction}
\label{sec:introduction}

One of the emerging applications of meta-learning \citep{bengio1990learning, schmidhuber1992learning, hochreiter2001learning} is the problem of personalization in federated learning settings \citep{smith2017federated, kairouz2019advances, li2020federated}.
Multiple recent works have explored the parallels between personalizing models to different users in a federated context and adapting models to different tasks in a multitask context \citep{chen2018federated, khodak2019adaptive, jiang2019improving, yu2020salvaging}.
While the initial results from these efforts are promising, there are still many open questions when it comes to applying meta-learning to personalization in federated settings.
In this work, we aim to understand which of the modern meta-learning algorithms, and under which conditions, are best suited for personalization.

It is tempting to extrapolate the performance of meta-learning methods on standard few-shot learning benchmarks \citep{lake2015human, vinyals2016matching, ravi2016optimization} to the federated learning setting.
Unfortunately, the training and evaluation routines used for benchmarking violate some of the modeling assumptions of federated learning and potentially other real-world scenarios.
Specifically, the current practice is to train meta-learning methods on a large number of programmatically constructed supervised few-shot tasks sampled from an underlying labeled meta-dataset and then evaluate them on a small set of test tasks~\citep{triantafillou2019metadataset}.
This approach implicitly assumes a full control over and an unrestricted access to the training data, allows to train on combinatorially many tasks that reuse the underlying labeled data and, as a result, ignores the associated labeling costs.
In personalization, however, when tasks correspond to different users and labels correspond to user-specific preferences, ratings, etc., the data is private and cannot be reused across multiple tasks and labeling user data can be quite costly, which makes existing evaluation practices often ill-suited.

For a learning-to-learn method to work well in federated settings, it must be data efficient and able to generalize to new tasks under a limit on the available supervision.
Our work is motivated by the current lack of basic understanding of generalization properties of modern meta-learning algorithms.
To this end, we analyze theoretically two major families of modern algorithms---\emph{gradient-based} (\maml~\citep{finn2017maml} and \reptile~\citep{nichol2018reptile}) and \emph{metric-based} (\protonets~\citep{snell2017prototypical})---and characterize how the number of training tasks and the number of labeled data points per task affect performance of each method.

To validate our theory and understand data efficiency of different methods in practice, we further introduce an alternative evaluation framework for meta-learning, where we measure performance as a function of the \emph{supervision budget} or the total amount of labeled data across training tasks.
Despite the conceptual simplicity, our framework allows to compare meta-learning algorithms under more realistic assumptions and reveals interesting and practically relevant tradeoffs.
Finally, to improve data efficiency a step further, we introduce \emph{active meta-learning}---a method-agnostic approach that extends meta-learning with active data selection at training time and yields improved empirical performance on the benchmarks under limited supervision.

\vspace{-1.75ex}
\paragraph{Contributions.}
\begin{enumerate}[topsep=-2pt,itemsep=0pt,leftmargin=22pt]
    \item We characterize data-efficiency of modern meta-learning methods theoretically using techniques from algorithmic stability \citep{bousquet2002stability, maurer2005algorithmic, hardt2015train} and provide generalization bounds that indicate how much supervision is needed and how it must be allocated for each method to attain the desired performance.
	\item We analyze \maml, \reptile, and \protonets experimentally both on the standard \omniglot and \miniimagenet meta-datasets as well as on federated learning benchmarks \citep{caldas2018leaf}.
	Our results support predictions of the stability theory, reveal the relative differences between the methods in the limited supervision regime, and provide insights into how to best allocate the available supervision.
	\item Finally, we benchmark meta-learning methods with and without active data selection at training time and demonstrate improved performance of active meta-learning under limited supervision.
\end{enumerate}

\section{Related Work}
\label{sec:related_work}

\paragraph{Meta-learning theory.}
Our analysis builds on the classical notion of algorithmic stability of \citet{bousquet2002stability} and extends the bounds of \citet{maurer2005algorithmic} to modern gradient-based and metric-based methods.
Recent theoretical work on meta-learning has largely focused on gradient-based methods in online settings \citep{khodak2019adaptive, khodak2019provable, finn2019online, denevi2019learning, denevi2019online}, providing sharper bounds but under stronger assumptions on smoothness and convexity than those required by our stability theory.
Several other works have studied convergence properties of gradient-based meta-learning from the optimization standpoint rather than generalization~\citep[\eg,][]{franceschi2018bilevel, fallah2020convergence}.
To the best of our knowledge, none of the previous work provides sufficient insight into the data efficiency aspects of modern meta-learning algorithms.

\vspace{-1.75ex}
\paragraph{Federated learning.}
While the classical problem of federated learning involves estimation of a single, global model from heterogeneous data~\citep{mcmahan2016communication}, many recent works have pointed out the growing importance of tailoring models to each individual user~\citep{smith2017federated, yu2020salvaging}.
Gradient-based meta-learning has been explored empirically~\citep{chen2018federated, jiang2019improving} and analyzed theoretically \citep{charles2020outsized} as a natural choice for this problem.
However, personalization in federated settings is still in a fairly nascent state~\citep{kairouz2019advances} and our work aims to make a step toward better understanding of meta-learning in this new context.

\paragraph{Active learning.}
Combinations of active and few-shot learning have been explored in prior work in multiple settings:
\citet{woodword2017activeoneshot} analyzed active-learning in a streaming setting.
\citet{boney2017semiactive} showed that active learning can improve performance at test time on new tasks, but did not consider active learning at meta-training time.
The works of \citet{bachman2017learning} and \citet{ravi2018batchactive} most closely resemble our setup, but neither approach was evaluated in the limited supervision regime or considered the problem of personalization in federated learning.

\section{Background}
\label{sec:background}

We start by introducing the preliminaries necessary to state our theoretical results as well as overview the key elements of modern meta-learning algorithms.\footnote{We assume that the reader is generally familiar with gradient-based and metric-based meta-learning~\citep{finn2017maml, nichol2018reptile, snell2017prototypical}.
Our overview is mainly focused on establishing the notation.}

\subsection{Meta-learning Formulation}
\label{sec:meta-learning-formulation}

Meta-learning operates in a multi-task setting, where the goal is to design a meta-algorithm $\Ab$ that can process data from multiple tasks $\{\Tc_1, \Tc_2, \ldots\}$ and output a learning algorithm $A$.
Given a task $\Tc_i$, the latter must be able to produce an accurate model for that task.
In the context of federate learning (FL), tasks correspond to users and are represented by their personal datasets.

A meta-algorithm is data-efficient if a small number of training tasks (with only few data points per task) is sufficient for it to produce a good learning algorithm.
Assuming that all tasks originate from a common underlying distribution, $\Tc \sim \Pb$, we are interested in \emph{meta-generalization} of $\Ab$, \ie, how many training tasks and how much data per task is necessary to ensure a certain level of performance on future tasks sampled from $\Pb$.

In this paper, we focus on \emph{few-shot classification}~\citep{vinyals2016matching}, where a learning task $\Tc_i$ is defined by a small i.i.d. sample of size $m$ (called the \emph{support set}~\citep{thrun2012learning}) that consists of $(x, y)$-pairs, $S_i := \{(x_j, y_j)\}_{j=1}^m \sim \Dc_i^m$, where $\Dc_i \sim \Pb$.
Formally, meta-learning can be formulated as a search problem over some family of algorithms $\Ac$:
\begin{align}
	\label{eq:transfer-risk}
    & \min_{A \in \Ac} \left\{ \Rc\left(A, \Pb\right) := \ep[\Dc \sim \Pb]{\ep[S \sim \Dc^m]{R(A(S), \Dc)}} \right\}, \\
    \label{eq:classical-risk}
    & \mathrm{where}\, R\left(f, D\right) := \ep[(x, y) \sim \Dc]{\ell\left(f(x), y\right)}.
\end{align}
The objective $\Rc(A, \Pb)$ given in \cref{eq:transfer-risk} is called the \emph{transfer risk}~\citep{baxter2000model} and is defined as the expected error encountered by models $f_i(\cdot) := A(S_i)$ produced by the learning algorithm $A$ on new tasks $\Tc_i$ sampled from $\Pb$.
Transfer risk characterizes how well an algorithm $A$ meta-generalizes over a task distribution $\Pb$.

The algorithms from family $\Ac$ are typically designed to be able to learn from the limited support data.
Meta-learning methods vary in how they define $\Ac$, which may consist of iterative optimization procedures \citep{ravi2016optimization, finn2017maml, nichol2018reptile}, approximate inference \citep{alshedivat2018continuous, grant2018recasting, garnelo2018conditional, finn2018probabilistic}, or nearest neighbor approaches~\citep{vinyals2016matching, snell2017prototypical}, among many other hybrid methods proposed in recent years~\citep[\eg,][]{rusu2018meta, chen2019closer, zintgraf2019fast}.
The methods also differ in terms of the objective functions they optimize to minimize the transfer risk \cref{eq:transfer-risk}, as the latter cannot be computed or optimized directly.

\textbf{Notation.}
Throughout the paper, we denote data samples with $S$ (or $Q$), learning algorithms with $A$, models that these algorithms produce after processing data samples with $f(\cdot)$ or $A(S)(\cdot)$; subscripts next to $A$ indicate the variables that parametrize algorithms.
Similarly, meta-samples (\ie, sets of data samples for multiple tasks) and meta-algorithms (\ie, procedures that return algorithms) are denoted with $\Sb$ and $\Ab$, respectively.
The number of training tasks is denoted by $n$ and the number of data points per task by $m$.

\subsection{Generalization in Meta-learning}
\label{sec:meta-generalization-error-bound}

The formulation of meta-learning given in \cref{eq:transfer-risk} was originally introduced by \citet{baxter2000model}, who derived the first bounds on the excess transfer risk (\ie, \emph{meta-generalization error bounds}).
For meta-algorithms $\Ab$ that produce learning algorithms $A$ by optimizing a loss $\Lc$, we can define meta-generalization as follows.
\vspace{-1ex}
\begin{definition}[Meta-generalization Error Bound]
    \label{def:meta-generalization-bound}
	Let $\Ab$ be a meta-algorithm which, given a meta-sample $\Sb := \{S_1, \dots, S_n\}$ from $n$ tasks, outputs an algorithm, $\Ab(\Sb) := \argmin_{A \in \Ac} \Lc(A, \Sb)$.
	A two-argument function $B(\delta, \Sb)$ is called a meta-generalization error bound for $\Ab$ if for any task distribution $\Pb$ and $\delta \in (0, 1]$, the following inequality holds with probability at least $1 - \delta$:
	\begin{equation}
		\label{eq:excess-tranfer-risk-bound}
		\Rc(\Ab(\Sb), \Pb) - \Lc(\Ab(\Sb), \Sb) \leq B(\delta, \Sb)
	\end{equation}
\end{definition}
\vspace{-1ex}
\noindent
\citet{maurer2005algorithmic} developed a general technique for obtaining such bounds $B(\delta, \Sb)$ using algorithmic stability~\citep{bousquet2002stability}.
In \cref{sec:analysis}, we will specialize Maurer's bounds to modern meta-learning algorithms.

\subsection{Modern Meta-learning Algorithms}
\label{sec:meta-learning-algorithms-background}

In this paper, we study three popular meta-learning methods that represent two broad categories: \emph{gradient-based} (\maml~\citep{finn2017maml} and \reptile~\citep{nichol2018reptile}) and \emph{metric-based} (\protonets~\citep{snell2017prototypical}).
We have selected these methods as many recent algorithmic developments in meta-learning, few-shot learning, and their applications are variations of those three~\citep[\eg,][]{mcmahan2016communication, duan2017one, li2017meta, rusu2018meta, zintgraf2019fast, chen2018federated}.
However, conclusions of our study are broadly applicable to the majority of modern meta-learning.

\textbf{Gradient-based meta-learning} defines the family of algorithms $\Ac$ as iterative optimization procedures: $A(\Sc) := \argmin_{\theta \in \Theta} L(f_\theta; S)$.
\maml and \reptile are gradient-based methods that approximately solve this minimization problem using an \emph{inner loop} of $T$ (stochastic) gradient steps with a learning rate $\alpha$ starting from a common initialization $\theta_0$ shared across tasks:
\begin{equation}
    \label{eq:gradient-based-adaptation}
    A_{\theta_0}(S) := f_{\theta_T},\, \mathrm{where}\, \theta_{t+1} := \theta_{t} - \alpha \nabla_{\theta_{t}} L(f_{\theta_{t}}; S).
\end{equation}
In this case, meta-learning amounts to search for an optimal initialization $\theta_0^\star$, which is done via the \emph{outer-loop} optimization of another objective function $\Lc(A_{\theta_0}; \Sb)$.
The key difference between \maml and \reptile is the loss $\Lc$ they optimize in the outer loop.
\reptile optimizes the empirical estimator of the transfer risk:%
\footnote{%
More precisely, \reptile updates $\theta_0$ iteratively: $\theta_0 \leftarrow \theta_0 + \varepsilon \frac{1}{n} \sum_{i=1}^n (\theta_i - \theta_0)$, where $\theta_i = A_{\theta_0}(S_i)$, $\varepsilon > 0$.
These updates approximately minimize $\Lc_\mathrm{emp}$ (see Appendix~\ref{app:reptile-objective-function}).}
\begin{align}
	\label{eq:empirical-objective}
	\Lc_\mathrm{emp}(A_{\theta_0}; \Sb) &:= \frac{1}{n} \sum_{i=1}^n \hat R(A_{\theta_0}, S_i), \\
	\hat R(A_{\theta_0}, S_i) &:= \frac{1}{|S_i|} \sum_{(x, y) \in S_i} \ell(A_{\theta_0}(S_i)(x), y),
\end{align}
while \maml holds out a sub-sample of $S$---called the \emph{query set}~\citep{finn2017maml}, which we denote $Q$---and optimizes an estimator of the transfer risk based on the held out set:
\begin{align}
	\label{eq:heldout-objective}
	\Lc_Q(A_{\theta_0}; \Sb) &:= \frac{1}{n} \sum_{i=1}^n \hat R_Q(A_{\theta_0}, S_i), \\
	\hat R_Q(A_{\theta_0}, S_i) &:= \frac{1}{|Q_i|} \sum_{(x, y) \in Q_i} \ell(A_{\theta_0}(S_i \setminus Q_i)(x), y).
\end{align}
As we will see, this difference in the meta-objectives will result in different meta-generalization bounds as well as different empirical behavior and practical implications.
We note that the most commonly used algorithm in federated learning~\citep[\textsc{FedAvg},][]{mcmahan2016communication} is identical to \reptile without fine-tuning at test time \citep{chen2018federated, khodak2019adaptive, jiang2019improving}.

\textbf{Metric-based meta-learning}
methods define the family of algorithms $\Ac$ that return non-parametric soft-nearest-neighbor models.
\protonets is one of such methods that computes \emph{prototype vectors} for each class in the inner loop and returns the following models:
\begin{align}
    \label{eq:proto-learning-algorithm}
    A_\theta(S)(x) &:= \argmax_{y \in \Yc} \frac{\exp(-d(\gv_\theta(x), \cv_y)}{\sum_{y^\prime \in \Yc} \exp(-d(\gv_\theta(x), \cv_{y^\prime})}, \\
    \cv_y &:= \frac{1}{|\Sc_y|} \sum_{(x, \cdot) \in \Sc_y} \gv_\theta(x), \quad \forall y \in \Yc
\end{align}
where the distance $d(\cdot, \cdot)$ is computed in the embedding space of $\gv_\theta(\cdot)$ and the so called \emph{class prototypes} $\cv_y$ are computed by averaging embeddings of the corresponding support points.
In the outer loop, \protonets optimize the same $\Lc_Q(A_\theta; \Sb)$ objective as \maml.

\clearpage
\section{Analysis}
\label{sec:analysis}

We adapt results from the stability theory of stochastic gradient methods~\citep{hardt2015train} and extend the classical bounds provided by \citet{maurer2005algorithmic} to \maml, \reptile, and \protonets.
We also make a few key observations about their expected behavior of these methods when the number of tasks and data points per task is limited, which is of practical importance to federated settings.
All proofs are provided in Appendix~\ref{app:proofs}.

\subsection{Understanding Meta-generalization via Algorithmic Stability}
\label{sec:stability-and-generalization}

As Definition~\ref{def:meta-generalization-bound} suggests, meta-generalization error is the discrepancy between the objective $\Lc(A; \Sb)$ optimized by a meta-learning method and the true transfer risk $\Rc(A, \Pb)$.
If the objective function is the empirical estimator $\Lc_\mathrm{emp}(A; \Sb)$, then following bound holds~\citep{maurer2005algorithmic}:
\vspace{-1ex}
\begin{equation}
	\label{eq:maurer-bound}
	\begin{aligned}
    	\MoveEqLeft \Rc(\Ab(\Sb), \Pb) - \Lc_\mathrm{emp}(\Ab(\Sb); \Sb) \\
    	&\leq 2\beta^\prime + (4n\beta^\prime + M)\sqrt{\frac{\ln(1/\delta)}{2n}} + 2\beta
	\end{aligned}
\end{equation}
with probability at least $1 - \delta$; $\beta^\prime$ and $\beta$ define stability\footnote{%
Intuitively, an algorithm (or meta-algorithm) is called stable if removing a single point from $S$ (or $\Sb$) would not affect its output by much. Precise definitions of algorithmic stability are provided in Appendix~\ref{app:proofs}.}
of the meta-learning algorithm $\Ab$ and of any learning algorithm $A$ it produces, respectively.
Generally, $\beta^\prime$ and $\beta$ are functions of the number of training tasks $n$ and data points per task $m$ and depend on the specific algorithms.
Maurer's bound becomes non-trivial when $\beta^\prime = o(1/n^a), a \geq 1/2$ and $\beta = o(1/m^b), b \geq 0$.

To derive bounds for modern meta-learning algorithms, we need two additional results.
First, for algorithms that optimize the Q-estimator $\Lc_Q(\Ab; \Sb)$ instead of the $\Lc_\mathrm{emp}(\Ab; \Sb)$ of the transfer risk, we need to bound on the difference $\Rc(\Ab(\Sb), \Pb) - \Lc_Q(\Ab(\Sb); \Sb)$.
We prove the following theorem that provides such a bound.
\vspace{-1ex}
\begin{theorem}
	\label{thm:meta-generalization-bound-rq}
	Let $\Ab$ be $\beta^\prime_Q$-uniformly stable with respect to $\hat R_Q$.
	Then, the following indequality holds for any task distribution $\Pb$ with probability at least $1 - \delta$:
	\vspace{-1ex}
	\begin{equation}
	    \label{eq:generalization-bound-loss-q}
	    \begin{aligned}
		\MoveEqLeft \Rc(\Ab(\Sb), \Pb) - \frac{1}{n}\sum_{i=1}^n \hat R_Q(\Ab(\Sb), S_i) \\
		&\leq 2 \beta_Q^\prime + (4n\beta_Q^\prime + M)\sqrt{\frac{\ln(1/\delta)}{2n}}
		\end{aligned}
	\end{equation}
	where $\hat R_Q(A, S_i) := \frac{1}{|Q_i|} \sum_{(x, y) \in Q_i} \ell(A(S_i \setminus Q_i)(x), y)$ with $\ell(\cdot, \cdot)$ bounded by $M$.
\end{theorem}
\vspace{-1ex}
Note that the bound in \cref{eq:generalization-bound-loss-q} lacks the term $2\beta$ which depends on the stability of the inner loop learning algorithm and is present in \cref{eq:maurer-bound}.

To be able to compare the generalization of \reptile, \maml, and \protonets, we need expressions for $\beta$, $\beta^\prime$, $\beta_Q^\prime$, as functions of the number of training tasks $n$ and the number of data points per task $m$.
Using results from stability theory of stochastic gradient method (SGM) due to \citet{hardt2015train} and bounds in \cref{eq:maurer-bound,eq:generalization-bound-loss-q}, we arrive at the following theorem.
\vspace{-1ex}
\begin{theorem}
	\label{thm:meta-generalization-bounds}
	Let the meta-algorithm $\Ab$ be an SGM that optimizes an $L^\prime$-Lipschitz and $\gamma^\prime$-smooth loss $\Lc(A, \Sb)$ by taking $T^\prime$ steps with non-increasing step sizes $\alpha^\prime_t \leq c^\prime / t$.
	With probability at least $1 - \delta$, we have the following:
	\begin{enumerate}[itemsep=-5pt,topsep=0pt,leftmargin=14pt]
		\item If $\Lc(A; \Sb)$ is Q-estimator of the transfer risk, then the following bound holds:\vspace{-1ex}
		\begin{equation}
			\label{eq:meta-generalization-bound-sgm-q}
			\Rc(A, \Pb) - \Lc(A; \Sb) \leq O\left(L^{\prime2}T^\prime\sqrt{\frac{\ln(1/\delta)}{n}}\right)
		\end{equation}
		\vspace{-2ex}
		\item If $\Lc(A; \Sb)$ is the empirical estimator of the transfer risk, the inner loop learning algorithm $A$ is an SGM that optimizes $L$-Lipschitz and $\gamma$-smooth loss $\ell(f(x), y)$ by taking $T$ steps with step sizes $\alpha_t \leq c / t$:\vspace{-1ex}
		\begin{equation}
		    \label{eq:meta-generalization-bound-sgm-emp}
    		\begin{aligned}
    			\MoveEqLeft \Rc(A, \Pb(\Tc)) - \Lc(A; \Sb) \\
    			&\leq O\left(L^{\prime 2}T^\prime\sqrt{\frac{\ln(1/\delta)}{n}} + L^2T\frac{1}{m}\right)
    		\end{aligned}
		\end{equation}
	\end{enumerate}
\end{theorem}
\vspace{-1ex}
The bound in \cref{eq:meta-generalization-bound-sgm-q} is applicable to \protonets and \maml; the one given in \cref{eq:meta-generalization-bound-sgm-emp} applies to \reptile.

\textbf{Observations.}
We can make the following observations by comparing expression in \cref{eq:meta-generalization-bound-sgm-q} and \cref{eq:meta-generalization-bound-sgm-emp}:
\begin{enumerate}[itemsep=0pt,topsep=-2pt,leftmargin=24pt]
    \item[\textbf{O1}.] When $n \rightarrow \infty$, the generalization error of any meta-learning method to which bound in \cref{eq:meta-generalization-bound-sgm-q} applies (\eg, \maml, \protonets) goes to 0.
    \item[\textbf{O2}.] The bound for \reptile has an additive term $O(L^2T/m)$ compared to \maml or \protonets.
    This implies that while we can reduce the generalization gap for \maml/\protonets by training on more tasks, \reptile always has a non-zero gap due to within-task sample complexity.
    \item[\textbf{O3}.] The bound in \cref{eq:meta-generalization-bound-sgm-q} may seem to be independent of the support set size $m$.
    This is unlikely,
    as the Lipschitz and smoothness constants of the $\Lc(A, \Sb)$ objective must depend on the properties of the support sets in $\Sb$, with larger sets more likely to results in better-behaved meta-objective.\footnote{%
    In practice, we observe that training \protonets and \maml on the same number of tasks but with larger support sets leads to better meta-test performance consistently, suggesting that larger support sets are generally better.}
    Our analysis suggests that the dependence of \cref{eq:meta-generalization-bound-sgm-q} on $m$ and $n$ is multiplicative rather than additive in \cref{eq:meta-generalization-bound-sgm-emp}, which means that a large enough $n$ can perhaps compensate for a small $m$ in \cref{eq:meta-generalization-bound-sgm-q}.
\end{enumerate}

\subsection{Implications for Meta-learning in the Limited Supervision Regime}
\label{sec:bounded-supervision}

The observations we have made in the previous section have important practical implications.

\vspace{-2ex}
\paragraph{Improving evaluation.}
To measure data-efficiency of meta-learning methods, as \textbf{O1} suggests, we should control for the number of tasks at meta-training time, since otherwise the observed differences in performance will \emph{not} be indicative of generalization.
However, all popular few-shot classification benchmarks based on \omniglot, \miniimagenet, and other datasets~\citep{triantafillou2019metadataset} train on tasks generated programmatically by randomly sampling support sets from a large underlying data pool.
Such a construction provides access to combinatorially many training tasks, virtually setting $n \rightarrow \infty$.
Instead of training on an endless stream of tasks, we propose an alternative evaluation scheme, where we strictly limit the number of unique tasks available at training time.
Not only our evaluation scheme corresponds to the standard FL setting, where we have a limited number of users for meta-training, but also is compatible with the popular few-shot learning benchmarking datasets.

\vspace{-2ex}
\paragraph{Understanding tradeoffs.}
In federated settings, acquisition of supervised data has an associated cost: in some cases, the number of users might be large, but tasks may require manual data labeling (\eg, prompt users about their preferences); in other cases, manual data labeling may not be necessary or expensive, but the number of unique users might be limited.
To select the best meta-learning method for a given problem, we need understand the tradeoffs.
\textbf{O2} suggests that using \reptile might be suboptimal if tasks have very small support sets; at the same time, \maml and \protonets are likely to work better when trained on more tasks with fewer labels allocated to each task.

\vspace{-2ex}
\paragraph{Optimally allocating the labeling budget.}
Even when the labeling budget, the number of tasks, and the support sets are all fixed, we often have additional flexibility in terms of which support points to label in each task.
The standard approach is to select these points uniformly at random.
Based on \textbf{O3}, we hypothesize that carefully selected support sets may lead to \emph{better-behaved} meta-losses and improve meta-generalization.

\begin{figure}[t]
\vspace{-2ex}
\begin{minipage}[t]{\columnwidth}
\begin{algorithm}[H]
    \small
    \caption{Active Meta-learning}
    \label{alg:active-meta-learning}
    \begin{algorithmic}[1]
    \INPUT $\Pb$: task distribution, $B$: labeling budget, $L$: \# labels per task, $A_\theta(\cdot)$: learning algorithm, $\texttt{select\_labeled}(\cdot)$: active labeling, $\texttt{meta\_update}(\cdot)$: meta-learning update.
    \REPEAT
    \STATE Initialize: $\theta \leftarrow \theta_0$, $D \leftarrow \varnothing$.
        \IF{$B > 0$ and time to get new tasks}
            \STATE Sample a new unlabeled set: $S_u \sim \Pb$.
            \STATE Initialize labeled support set: $S_l \leftarrow \varnothing$.
            \FOR{l in $1 \dots L$}
                \STATE $S_l \leftarrow S_l \cup \texttt{select\_labeled}(D, A_\theta(S_l))$.
            \ENDFOR
            \STATE Update training tasks: $D \leftarrow D \cup \{S_L\}$.
            \STATE Reduce available budget: $B \leftarrow B - L$.
        \ENDIF
        \STATE Sample a batch of tasks: $\{S_i\} \sim D$.
        \STATE Meta-learn: $\theta \leftarrow \texttt{meta\_update}(A_\theta, \{S_i\})$.
    \UNTIL{Convergence}
    \OUTPUT Meta-learned algorithm: $A_{\theta^\star}$.
    \end{algorithmic}
\end{algorithm}
\vspace{-6ex}
\begin{algorithm}[H]
    \small
    \caption{Active Label Selection}
    \label{alg:active-sampling}
    \begin{algorithmic}[1]
    \INPUT $S_u$: unlabeled set, $\theta_0$: initial parameters,\\$\gv_\theta(\cdot)$: embedding function, $f_\theta(\cdot)$: predictive model, $A_\theta(\cdot)$: learning algorithm, $L$: \# labels to sample.
    \STATE Initialize: $\theta \leftarrow \theta_0$.
    \STATE Compute representations: $\rv_i \leftarrow \gv_\theta(x_i),  x_i \in S_u$
    \STATE Compute predictive entropy: $h_{i} \leftarrow \Hc(f_\theta(x_i)), x_i \in S_u$
    \STATE Get clusters: $C \leftarrow \texttt{k-means++}(\{\hv_i\})$.
    \STATE Initialize the labeled set: $S_l \leftarrow \varnothing$.
    \FORALL {$c_j \in C$}
    \STATE Sample $L/|C|$ points from cluster $c_j$:\\
    $\{i_1, \dots, i_{L/|C|}\} \sim \text{Categorical}\left(\{h_i\}_{i \in c_j}\right)$
    \STATE Request labels: $S_l \leftarrow S_l \cup \{(x_{i_k}, y_{i_k})\}_{k=1}^l$.
    \STATE Update models: $\gv_\theta, f_\theta \leftarrow A_\theta(S_l)$.
    \ENDFOR
    \OUTPUT Labeled set: $\{(x_{i_1}, y_{i_1}), \dots, (x_{i_L}, y_{i_L})\}$
    \end{algorithmic}
\end{algorithm}
\end{minipage}
\vspace{-2ex}
\end{figure}

\section{Active Meta-learning Algorithms}
\label{sec:methods}

We conjecture that actively selecting labeled support points can potentially improve performance of meta-learning methods in the limited supervision regime.
Assuming that we are given a hard labeling budget $B$ and a set of training tasks with \emph{unlabeled} support sets, we propose an algorithm that allocates this budget among the tasks ($m$ points per task) during meta-training by adaptively selecting which points to label.

\paragraph{Active meta-learning.}
We propose \cref{alg:active-meta-learning} for \emph{active meta-learning}, which gradually acquires labels for selected support points at meta-training time and proceeds in 3 steps:
1) start with a fully unlabeled support set (lines 4-5),
2) run an active label selection sub-routine that selects and labels a few points from the support set (line 6-8),
3) add the task with the labeled support set to the growing collection of training tasks (lines 9-10).
Importantly, the active sampling procedure is integrated into the inner loop, where our approach interleaves active labeling with model adaptation.
Put differently, instead of selecting support points that are labeled all at once, we sample them in mini-batches and re-adapt the model on already sampled points before requesting the next batch (Algorithm~\ref{alg:active-sampling}, lines 6-10).

\vspace{-2ex}
\paragraph{Active labeling.}
\cref{alg:active-meta-learning} relies on active labeling as a subroutine.
We designed a very simple active labeling method (Algorithm~\ref{alg:active-sampling}), which uses model uncertainty and data diversity as the selection criteria, which are most common in the active learning literature~\citep{settles2009active}.
Given an unlabeled support set for a new task, first, we compute representations for each point by extracting them form a hidden layer of the current model and clusters them using \texttt{k-means++} \citep{arthur2006k}.
Next, we compute predictive probabilities of the current model for each unlabeled point and select points that will be labeled proportionally to the model's uncertainty with stratification by cluster.
We use the entropy of the model's predictive distribution as the measure of model's uncertainty.
Our approach is simple, works well, and is essentially a combination of uncertainty- and diversity-based active label acquisition~\citep{huang2010active, zhdanov2019diverse}.

\section{Experiments}
\label{sec:experiments}

Our experimental analysis is organized into two parts.
In the first part, we validate the predictions of our theory by analyzing behavior of \maml, \reptile, and \protonets under different supervision tradeoffs on the standard few-shot learning datasets.
In the second part, we benchmark these methods on few-shot and federate learning datasets with a fixed total labeling budget and random versus active data labeling.

\begin{figure}[t]
    \centering
    \includegraphics[width=\columnwidth]{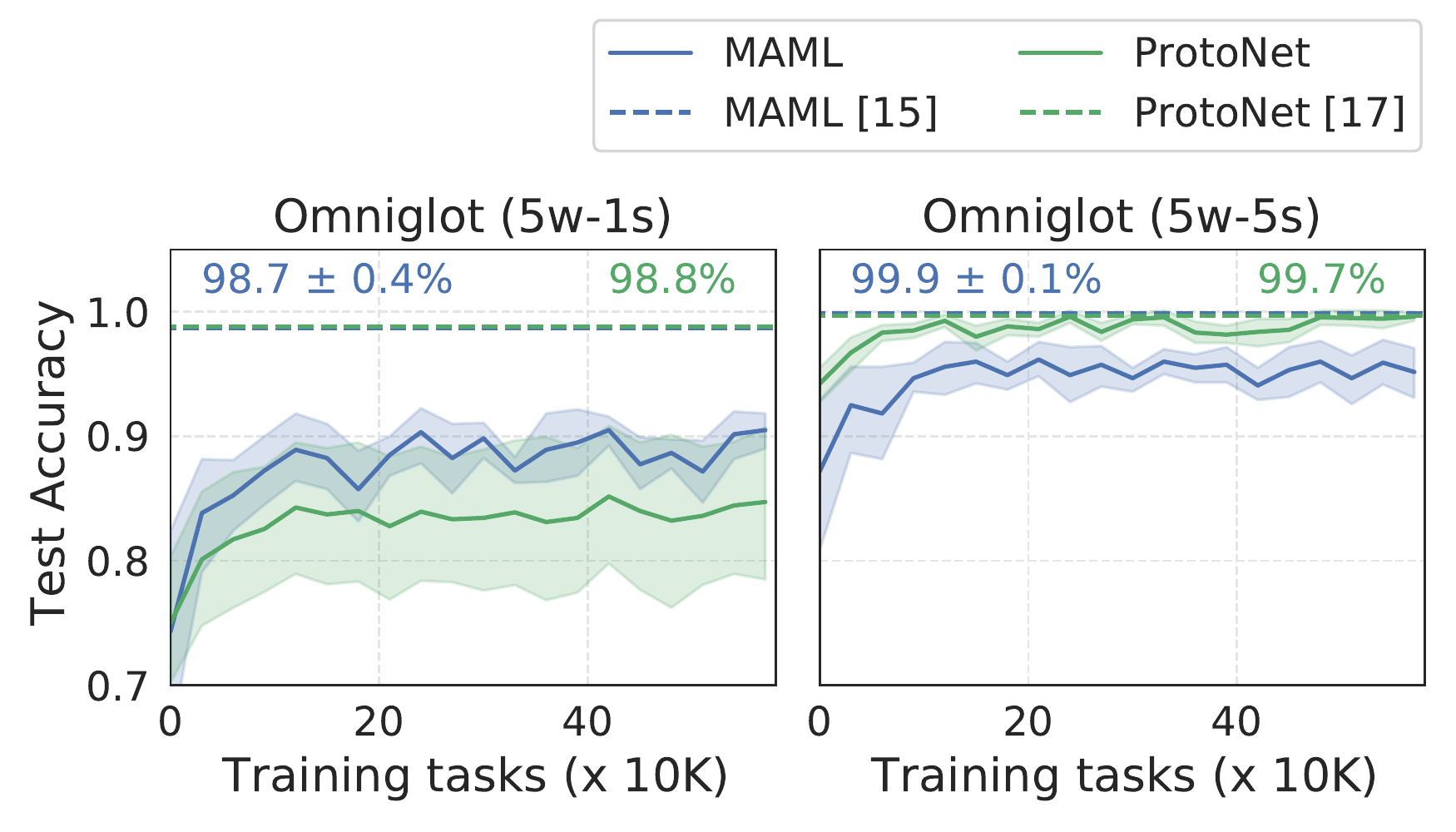}
    \caption{%
    Test accuracy of \maml and \protonets trained under different constraints on the number of training tasks.
    The difference between the method is only visible when we control for the number of training tasks.
    Shaded regions are 95\% CI (based on 3 runs with different seeds).}
    \label{fig:results-maml-vs-proto}
    \vspace{-2ex}
\end{figure}

\subsection{The Setup}
\label{sec:exp-setup}

\paragraph{Datasets.}
We conduct our study on \omniglot \citep{lake2015human} and \miniimagenet \citep{vinyals2016matching} with the standard data splits into train, validation, and test, and federated EMNIST dataset~\citep{caldas2018leaf}.
We consider 5-way and 20-way classification tasks with 1-5 support shots and 1 query shot for \omniglot and \miniimagenet.
For EMNIST, each task is a 62-way classification with the data corresponding to a unique user (all 3400 users were split into train/validation/test sets as 3000/200/200), where we similarly limit the support data to 1-5 shots and use the full tests sets of each user as query sets.
Note that our personalized federated learning setup slightly differs from the standard FL benchmarks where the amount of support data per user is not restricted to 1-5 points.
No data augmentation is used in any of our settings.

\paragraph{Models and methods.}
We consider 3 meta-learning methods: \maml, \reptile, and \protonets, using the standard hyperparameters and small convolutional backbone networks~\citep{finn2017maml, nichol2018reptile, snell2017prototypical}.

\vspace{-2ex}
\paragraph{Labeling budgets and strategies.}
In \omniglot and \miniimagenet, each $k$-way task is constructed by first selecting $k$ handwritten characters as classes, then selecting a 1-shot query set and a small support set from the corresponding data.
In EMNIST, the classes are fixed for all tasks, each task corresponds to a user, the support sets are selected from the users' training data and test data is used as query sets.
The labeled points in the support sets are either selected uniformly at \emph{random} or \emph{actively} using \cref{alg:active-sampling}.

\vspace{-2ex}
\paragraph{Evaluation.}
All methods are trained either (a) in the \emph{classical regime}, where we do not control for the the number of training tasks or the total amount of labeled data, or (b) in the \emph{limited supervision regime}, where the total amount training labeled data is limited.
We report performance in terms of accuracy in each setting, denoted \texttt{[dataset] (Xw-Ys) (@ Z)}, where \texttt{X} is the number of ways, \texttt{Y} is the number of shots, and \texttt{Z} is the labeling budget.\footnote{The \texttt{Z} is not specified for the classical training regime.}
For each labeling budget, we report the accuracy on the test tasks for the method with hyperparameters selected based on the accuracy on the validation tasks.
Each of our experiments was repeated 3 times with different random seeds.

\vspace{-2ex}
\paragraph{Reproducibility.}
To support reproducibility, we have developed a software framework that allows users benchmark arbitrary meta-learning methods under different supervision tradeoffs.
The code and configurations for all our experiments will be released.

All additional details on the setup are in \cref{app:exp-details}.

\begin{figure*}[t]
    \centering
    \includegraphics[width=\textwidth]{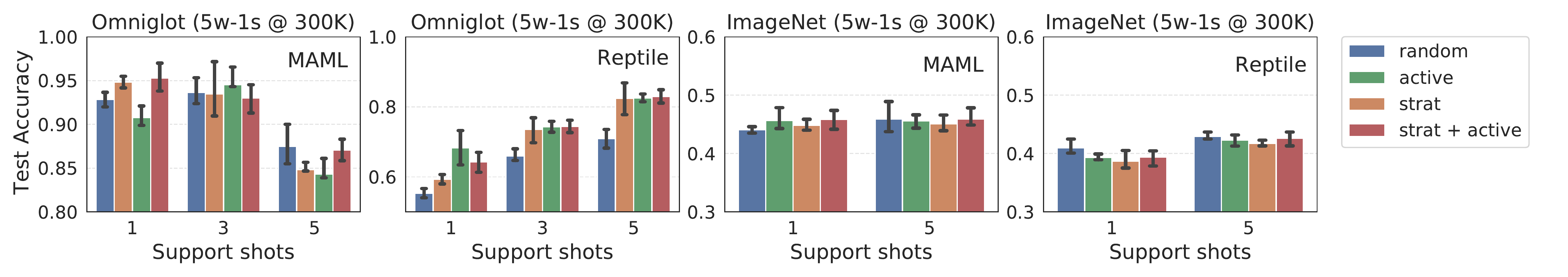}
    \caption{%
    Test accuracy of \maml and \reptile trained on \omniglot and \miniimagenet under a limit on the available supervision (the total labeling budget was fixed at 300K) as a function of the number of support shots.
    \maml and \reptile exhibit visibly different behaviors.
    Error bars indicate 95\% CI (based on 3 runs with different random seeds).}
    \label{fig:results-all-bounded}
\end{figure*}
\begin{figure*}[t]
    \centering
    \includegraphics[width=\textwidth]{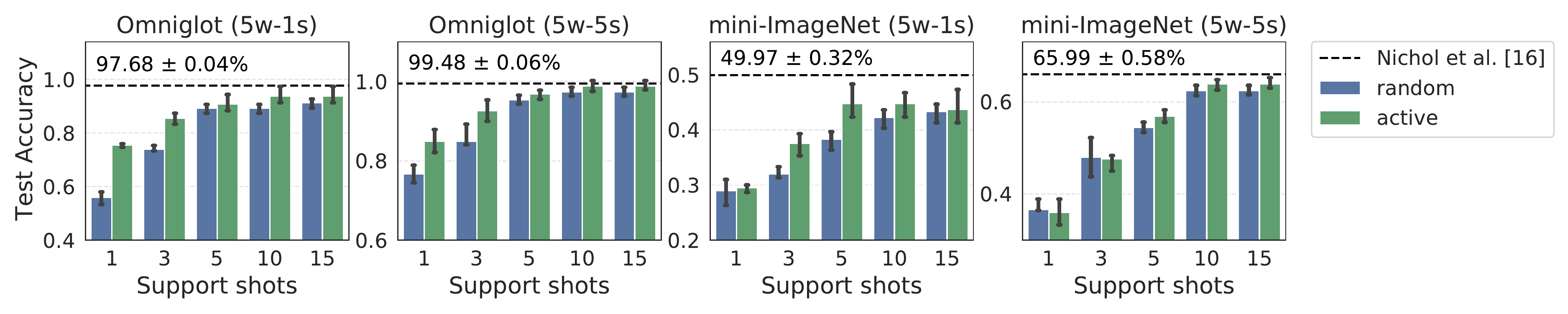}
    \caption{%
    Test accuracy of \reptile meta-trained for 100K steps without controlling for the number of training tasks (\ie, classical evaluation regime) as a function of the support set size of the training tasks.
    The dashed line corresponds to test accuracy of \reptile reported by \citet{nichol2018reptile} (trained in the same regime with 10 support shots and data stratification by class ).
    Error bars indicate 95\% CI (based on 3 runs with different random seeds).}
    \label{fig:results-reptile}
\end{figure*}

\subsection{Understanding Supervision Tradeoffs}
\label{sec:exp-understanding}

To validate the implications of our theory (\textbf{O1-3}), we conduct the following study.
First, we analyze the differences in behavior of meta-learning methods when the number of training tasks is limited.
Next, we analyze how the tradeoff between the number of tasks and the number of data points per task affect performance of different methods.
Finally, we discuss the effect of active data labeling on meta-learning.

\vspace{-2ex}
\paragraph{Classical evaluation vs. limited supervision.}
We ask whether different meta-learning methods behave differently when trained and evaluated in the classical vs. limited supervision regime.
To answer this question, we trained \maml and \protonets on 5-way, 1-shot and 5-shot \omniglot benchmark under different limits on the number of training tasks (ranging between 1-500K).
\cref{fig:results-maml-vs-proto} shows that the performance of \maml and \protonets is virtually indistinguishable in the classical evaluation regime.\footnote{When the number of training tasks is not limited, 10M+ unique tasks are typically generated during training.}
However, once the number of training tasks is limited, we observe quite a distinct behavior---\maml works better when trained on 1-shot tasks, while \protonets dominate the 5-shot benchmark.
Moreover, as the number of tasks increases ($n \rightarrow \infty$ in the limit), the generalization error reduces at a similar rate for both methods and the gap between them shrinks, as suggested by our theory.

\textbf{Exploring tradeoffs under limited supervision.}
In our next set of experiments, given a fixed labeling budget, we would like to find out how different ways of allocating supervision across tasks affects performance.
To this end, we fixed the labeling budget at 300K, and constructed meta-datasets based on \omniglot and \miniimagenet that satisfied the limit on the labeling budget.
Each meta-dataset consisted of 5-way tasks with 1-shot, 3-shot, or 5-shot support sets that were selected using different labeling strategies.
Due to fixed labeling budget, settings with larger support sets had fewer training tasks.
In addition to selecting labels uniformly at random and actively, we also conducted experiments with selection stratified by class.\footnote{
Stratification ensured that the selected support sets where class-balanced.
Although typical in meta-learning, such an approach requires knowing all labels in the support set \emph{a priori}, narrowing the scope of possible applications.}
The results for \maml and \reptile are presented in \cref{fig:results-all-bounded}.

First, we notices that \maml is able to attain much better performance when trained on more tasks with fewer labels per task, suggesting that the number of tasks is indeed the dominant factor that determines how well the method generalizes.
Interestingly, \reptile exhibits quite the opposite behavior and yields strictly better performance when trained on overall fewer tasks with more data points each.
This is consistent with the meta-generalization error bounds given in \cref{thm:meta-generalization-bounds} which has an additional $O(1/m)$ term for \reptile that turns out to be dominant in this particular case.
    
Second, active selection of labeled data is beneficial for both methods, although \reptile benefits from it more.
Theoretically, we hypothesise (and conjecture) that the effect might be due to improved constants in each terms of the meta-generalization bounds (\ie, due to ``better-behaved'' inner- and outer-loop objectives, when the labeled points are selected to minimize the uncertainty of the adapted model).
Note that the effects of active label selection and more data per task on \reptile are also visible in the classical evaluation regime (see \cref{fig:results-reptile}, where the number of training tasks was not controlled).

Taken together, our observations suggest that there are multiple factors that significantly affect performance of meta-learning in \emph{different ways}, when the availability of training data is limited.
While some of the effects of limited supervision can be reasonably explained by the stability theory, from a practical point of view, having benchmarks that can capture such effects is essential.

\subsection{Benchmarking @ Fixed Labeling Budgets}
\label{sec:exp-benchmarking}

\begin{table}[t]
    \centering
    \caption{%
    Results on the suite of bounded supervision benchmarks for \reptile, \maml, and \protonets trained using random (R) or active (A) selection of the labeled support.
    Each row in the table corresponds to a benchmark.}
    \label{tab:results}
    \vspace{-1ex}
    \renewcommand{\arraystretch}{1.2}
    \fontsize{8pt}{8pt}\selectfont
    \begin{tabu} to \linewidth {X[0.2]X[2.1]X[1.9,r]|X[r]X[r]|X[r]X[r]|X[r]X[r]}
        \toprule
        \multicolumn{3}{c|}{\textbf{Benchmark}} &
        \multicolumn{2}{c|}{\textbf{\reptile}} &
        \multicolumn{2}{c|}{\textbf{\maml}} &
        \multicolumn{2}{c}{\textbf{\proto}} \\
        \multicolumn{2}{l}{Dataset} & Budget
        & R & A
        & R & A
        & R & A \\
        \midrule
        \multirow{8}{*}{\rotatebox{90}{\omniglot}}
            & 5w-1s & 30K
            & 73.3 & 77.2                       %
            & 83.3 & 79.6                       %
            & \textbf{84.8} & 84.5  \\          %
            &       & 300K
            & 76.8 & 80.1                       %
            & 91.8 & 93.6                       %
            & 92.2 & \textbf{94.3}   \\         %
            \cmidrule{3-9}
            & 5w-5s & 30K
            & 89.3 & 88.0                       %
            & 91.5 & 90.2                       %
            & 95.1 & \textbf{96.8}   \\         %
            &       & 300K
            & 92.4 & 93.0                       %
            & 96.3 & 96.9                       %
            & 97.0 & \textbf{98.1}   \\         %
        \cmidrule{2-9}
            & 20w-1s & 60K
            & 68.4 & 67.9                       %
            & 79.0 & 79.8                       %
            & 84.9 & \textbf{85.4}   \\         %
            &        & 600K
            & 76.3 & 77.8                       %
            & 84.8 & 83.9                       %
            & 86.1 & \textbf{87.9}   \\         %
            \cmidrule{3-9}
            & 20w-5s & 60K
            & 89.8 & 92.4                       %
            & 93.3 & 94.1                       %
            & 95.2 & \textbf{96.0}   \\         %
            &        & 600K
            & 92.0 & 94.1                       %
            & 95.5 & 95.9                       %
            & \textbf{96.8} & \textbf{96.8} \\  %
        \midrule
        \multirow{4}{*}{\rotatebox{90}{\imagenet}}
            & 5w-1s & 30K
            & 38.5 & 39.4                       %
            & \textbf{44.2} & 44.0              %
            & 42.9 & 40.7            \\         %
            &       & 300K
            & 42.0 & 41.9                       %
            & \textbf{45.4} & 45.2              %
            & 39.9 & 41.0            \\         %
            \cmidrule{3-9}
            & 5w-5s & 30K
            & 55.4 & 55.2                       %
            & 56.8 & \textbf{57.1}              %
            & 54.0 & 53.4            \\         %
            &       & 300K
            & 56.7 & 59.4                       %
            & 60.8 & \textbf{61.0}              %
            & 57.4 & 55.8            \\         %
        \midrule
        \multirow{6}{*}{\rotatebox{90}{EMNIST}}
            & 62w-1s & 10K
            & 69.6 & 67.9                       %
            & 73.9 & 70.7                       %
            & \textbf{74.0} & 73.3   \\         %
            &        & 50K
            & 77.5 & 75.9                       %
            & \textbf{80.3} & 79.1              %
            & 78.2 & 78.0            \\         %
            &        & 100K
            & 82.0 & 84.5                       %
            & 86.7 & 87.7                       %
            & 88.8 & \textbf{90.1}   \\         %
            \cmidrule{3-9}
            & 62w-5s & 10K
            & 78.1 & \textbf{80.1}              %
            & 68.3 & 75.0                       %
            & 71.2 & 72.9            \\         %
            &        & 50K
            & 84.3 & \textbf{85.8}              %
            & 76.0 & 79.4                       %
            & 77.1 & 76.8            \\         %
            &        & 100K
            & 87.1 & \textbf{89.0}              %
            & 82.5 & 86.3                       %
            & 86.1 & 88.5            \\         %
            \bottomrule
    \end{tabu}
\end{table}

\begin{table}[t]
    \centering
    \caption{%
    \fedavg vs. \reptile on \emnist with random or active labeling.
    The labeling budget was fixed to 100K.
    Error bars indicate 95\% CI (3 runs, different random seeds).}
    \label{tab:fedavg-vs-reptile}
    \vspace{-1ex}
    \renewcommand{\arraystretch}{1.2}
    \fontsize{8pt}{8pt}\selectfont
    \begin{tabu} to \linewidth {X|X[r]|X[r]|X[r]|X[r]}
        \toprule
                            & \multicolumn{2}{c|}{\textbf{\fedavg}} & \multicolumn{2}{c}{\textbf{\reptile}}         \\
        \textbf{\# shots}   & Random            & Active            & Random            & Active                    \\
        \midrule
        1                   & $77.1 \pm 2.0$    & $77.8 \pm 2.1$    & $82.2 \pm 1.6$    & $\mathbf{84.2} \pm 1.9$   \\
        5                   & $80.7 \pm 2.4$    & $79.6 \pm 1.8$    & $87.0 \pm 2.1$    & $\mathbf{89.1} \pm 1.5$   \\
        \bottomrule
    \end{tabu}
    \vspace{-1ex}
\end{table}

There are many tradeoffs that we need to consider and balance when selecting a meta-learning algorithm for a setting with limited labeling budget.
To enable fair and reproducible comparison of meta-learning methods in different regimes, we introduce a suite of new benchmarks that test performance at different fixed labeling budgets.
Our benchmarks are compatible with both few-shot learning and federated learning datasets.

\cref{tab:results} presents results for each pair of meta-learning method and labeling strategy (random vs. active).
For each pair, we selected the best performance across settings with 1, 3, and 5 support shots at training time, and report results on both 1-shot and 5-shot test tasks.

\vspace{-2ex}
\paragraph{\omniglot and \miniimagenet.}
Comparing random versus active data labeling, we observe that active selection almost always consistently improves performance.
Overall, \protonets dominate other methods on \omniglot; \maml does better on \miniimagenet.
\reptile performs significantly worse than other methods when a strict limit on the labeled data is enforced.

\paragraph{Federated EMNIST.}
We observe overall similar trends for all methods evaluated in our personalized federated learning setting.
Interestingly, however, \reptile (with active data selection) tends to dominate other methods on 5-shot version of this benchmark.
This is likely due to the fact that the query sets in EMNIST larger than for \omniglot and \miniimagenet, which benefits \reptile since it does not distinguish between support and query sets and uses all available labeled data for inner loop updates.

For completeness purposes, we also compare \reptile and \fedavg~\citep{mcmahan2016communication} (the most popular FL algorithm used in many practical settings) and present results in \cref{tab:fedavg-vs-reptile}.
These algorithms differ only in whether or not they fine-tune the model on the support data at test time.
\reptile uses fine-tuning and clearly outperforms \fedavg on our personalized FL benchmarks.

\section{Conclusion}
\label{sec:conclusion}

Motivated by use-cases of meta-learning for personalization in federated learning, we have analyzed theoretically and experimentally the data efficiency of two major families of modern meta-learning algorithms.

Using stability theory, we derived bounds on the transfer risk (or meta-generalization error).
Our bounds revealed that:
1) \reptile and other methods designed to meta-learn by optimizing the empirical estimator of the transfer risk do not work well unless each training task contains a sufficient number of labeled points;
2) \protonets, \maml, and other methods designed to meta-learn by optimizing held out set losses can effectively learn from data-scarce tasks but require a large number of such tasks to meta-generalize.

Further, through multiple experiments, we confirmed predictions of our theory as well as studied behavior of popular meta-learning algorithms under different supervision tradeoffs that have important practical implications.
To that end, we introduced a new approach to benchmarking meta-learning methods in the limited supervision regime, which is compatible with arbitrary few-shot and federated learning datasets.

Finally, we conjectured that selecting labeled support sets at meta-training actively can improve performance of meta-learning methods in the limited supervision regime.
To test that hypothesis, we proposed active meta-learning---a simple approach that incorporated active labeling into the inner-loop.
Our method turned out to be quite effective, leading to improved performance of multiple methods under limited supervision.

We hope that this study will further accelerate research and enable wider adoption of meta-learning in personalized federated learning and other practical settings.

\subsection*{Acknowledgments}
The authors thank Mikhail Khodak, Harri Edwards, Afshin Rostamizadeh, Jenny Gillenwater for helpful comments on early versions of this manuscript.
This work was supported in part by DARPA FA875017C0141, the National Science Foundation grants IIS1705121 and IIS1838017, an Amazon Web Services Award, a JP Morgan A.I. Research Faculty Award, a Carnegie Bosch Institute Research Award, a Facebook Faculty Research Award, and a Block Center Grant.
MA was supported by Google PhD Fellowship.
LL contributed to this work while at CMU.
Any opinions, findings and conclusions or recommendations expressed in this material are those of the author(s) and do not necessarily reflect the views of DARPA, the National Science Foundation, or any other funding agency.

\bibliography{references}
\bibliographystyle{unsrtnat}

\clearpage
\onecolumn
\appendix
\section{Reptile Optimizes $\Lc_\text{emp}(A; \Sb)$}
\label{app:reptile-objective-function}

The Reptile meta-learning algorithm~\citep{nichol2018reptile} is defined as follows.
Given a model $f_\theta$ parametrized by $\theta$, it defines the inner loop algorithm $A_{\theta_0}$ as a $T$-step stochastic gradient optimization:
\begin{equation}
    \label{eq:reptile-inner-loop}
    A_{\theta_0}(S_i) := \theta_T^i, \quad \text{where} \quad \theta_{t+1}^i := \theta_t^i - \alpha_t \nabla_{\theta_t^i} \left( \frac{1}{|S_i|} \sum_{(x, y) \in S_i} \ell(f_{\theta_t}(x), y) \right)
\end{equation}
In the outer loop, it updates $\theta_0$, which is shared across all tasks, as follows:
\begin{equation}
    \label{eq:reptile-outer-loop}
    \theta_0 \leftarrow \theta_0 - \varepsilon \frac{1}{n} \sum_{i=1}^n (\theta_0 - A_{\theta_0}(S_i)), \varepsilon > 0
\end{equation}
We argue that Reptile outer loop updates \eqref{eq:reptile-outer-loop} approximate gradient descent on the empirical estimator of the transfer risk:
\begin{equation}
    \label{eq:empirical-transfer-risk}
    \Lc_\mathrm{emp}(A_{\theta_0}; \Sb) := \frac{1}{n} \sum_{i=1}^n \hat R(A_{\theta_0}, S_i), \quad \hat R(A_{\theta_0}, S_i) := \frac{1}{|S_i|} \sum_{(x, y) \in S_i} \ell(A_{\theta_0}(S_i)(x), y)
\end{equation}
To understand why this is the case, first, consider the gradient of $\Lc_\mathrm{emp}(A_{\theta_0}; \Sb)$ with respect to $\theta_0$:
\begin{align}
    \label{eq:empirical-transfer-risk-gradient}
    \nabla_{\theta_0} \Lc_\mathrm{emp}(A_{\theta_0}; \Sb)
    &= \frac{1}{n} \sum_{i=1}^n \nabla_{\theta_0} \hat R(A_{\theta_0}, S_i) \\
    &= \frac{1}{n} \sum_{i=1}^n \frac{1}{|S_i|} \sum_{(x, y) \in S_i} \nabla_{\theta_0} \ell(A_{\theta_0}(S_i)(x), y) \\
    &= \frac{1}{n} \sum_{i=1}^n \frac{1}{|S_i|} \sum_{(x, y) \in S_i} \nabla_{\theta_T^i} \ell(f_{\theta_T^i}(x), y) [\nabla_{\theta_0} A_{\theta_0}(S_i)]
\end{align}
Now, we can compute the difference between $\nabla_{\theta_0} \hat R(A_{\theta_0}, S_i)$ and the Reptile update $A_{\theta_0}(S_i) - \theta_0$:
\begin{equation}
    \label{eq:empirical-transfer-risk-gradient-reptile-diff}
    (\theta_0 - A_{\theta_0}(S_i)) - \nabla_{\theta_0} \hat R(A_{\theta_0}, S_i) =\, \frac{1}{|S_i|} \sum_{(x, y) \in S_i} \left[ \left( \sum_{t=1}^T \alpha_t \nabla_{\theta_t^i} \ell(f_{\theta_t^i}(x), y) \right) - \nabla_{\theta_T^i} \ell(f_{\theta_T^i}(x), y) [\nabla_{\theta_0} A_{\theta_0}(S_i)] \right]
\end{equation}
Expression in the square brackets is the difference between $T$ inner loop gradient steps on $\ell(f_\theta(x), y)$ and the gradient at the final $T$-th step transformed by the Jacobian $\nabla_{\theta_0} A_{\theta_0}(S_i)$.
This expression was analyzed by \citet{nichol2018reptile} using perturbation theory and Taylor approximation, where it was shown that this difference is equal to the following:
\begin{equation}
    \label{eq:empirical-transfer-risk-gradient-reptile-diff-i}
    \left( \sum_{t=1}^T \alpha_t \nabla_{\theta_t^i} \ell(f_{\theta_t^i}(x), y) \right) - \nabla_{\theta_T^i} \ell(f_{\theta_T^i}(x), y) [\nabla_{\theta_0} A_{\theta_0}(S_i)] = (I - \alpha H_{\theta_0}^T) \sum_{t=1}^{T-1} \nabla_{\theta_t^i} \ell(f_{\theta_t^i}(x), y) + O(\alpha^2)
\end{equation}
where $H_{\theta_0}^T$ is the Hessian of $\ell(f_{\theta_T}(x), y)$ at $\theta_0$, $\alpha := \max_{t \in [1, T]} \alpha_t$.

Assuming that $\alpha$ (\ie, the inner loop step size) is sufficiently small and the norm of $\nabla_\theta \ell(f_\theta(x), y)$ is bounded by some constant $G$, the difference in \eqref{eq:empirical-transfer-risk-gradient-reptile-diff-i} is bounded by $(1 - \alpha \lambda_\mathrm{max}(H_{\theta_0}^T)) G(T - 1) + O(\alpha^2)$, which implies:
\begin{equation}
    (\theta_0 - A_{\theta_0}(S_i)) - \nabla_{\theta_0} \hat R(A_{\theta_0}, S_i) \leq (1 - \alpha \lambda_\mathrm{max}(H_{\theta_0}^T)) G(T - 1) + O(\alpha^2)
\end{equation}
The deviation between $\nabla_{\theta_0} \Lc_\mathrm{emp}$ and Reptile updates would be small when the inner loop objective is well behaved (has a small $G$), the number of inner loops steps $T$ is not too large and the step sizes $\alpha$ are small, which would ensure convergence of Reptile to a stationary point of $\Lc_\mathrm{emp}$.
We leave sharper analysis of the convergence rates in the case of non-convex and convex $\ell(\cdot, \cdot)$ to future work.

\section{Proofs}
\label{app:proofs}

In this section, we provide detailed expressions for meta-generalization bounds, proofs for Theorems 2 and 3, statements (and proof sketches where necessary) for classical auxiliary results, and further discuss the implications and limitations of our analysis.

\subsection{Classical Bounds on Meta-generalization Error}
\label{app:meta-generalization-error-bound}

The meta-generalization bounds provided in Section~4 directly extend of the following classical result by \citet{maurer2005algorithmic} (which in turn uses meta-learning formulation of \citet{baxter2000model} and is a direct adaptation of the algorithmic stability bounds of \citet{bousquet2002stability}).

\begin{theorem}[Theorem 1 from \citep{maurer2005algorithmic}]
	\label{thm:meta-generalization-maurer}
	Let the meta-algorithm $\Ab$ satisfy the following two conditions:
	\begin{enumerate}[itemsep=1pt,topsep=0pt,leftmargin=21pt]
		\item[C1.] For every pair of meta-samples $\Sb = \{S_1, \dots, S_n\}$, $\Sb^{-i} := \Sb \setminus \{S_i\}$, and for any sample $S$, we have $|\hat R(\Ab(\Sb), S) - \hat R(\Ab(\Sb^{-i}), S)| \leq \beta^\prime$.
		\item[C2.] For any pair of samples $S = \{(x_1, y_1), \dots, (x_m, y_m)\}$, $S^{-j} := S \setminus \{(x_j, y_j)\}$, any algorithm $A$ produced by $\Ab$, and any $(x, y)$, we have $|\ell(A(S)(x), y) - \ell(A(S^{-j})(x), y)| \leq \beta$.
	\end{enumerate}
	Then for any task distribution $\Pb(\Tc)$, with probability at least $1 - \delta$ the following inequality holds:
	\begin{equation}
		\label{eq:maurer-bound-app}
		\Rc(\Ab(\Sb), \Pb(\Tc)) - \Lc_\mathrm{emp}(\Ab(\Sb); \Sb) \leq 2\beta^\prime + (4n\beta^\prime + M)\sqrt{\frac{\ln(1/\delta)}{2n}} + 2\beta,
	\end{equation}
	where $\Lc_\mathrm{emp}(\Ab(\Sb); \Sb) := \frac{1}{n} \sum_{i=1}^n \hat R(\Ab(\Sb), S_i)$, $\hat R(A, S_i) := \frac{1}{|S_i|} \sum_{(x, y) \in S_i} \ell(A(S_i)(x), y)$ with the loss function $\ell(\cdot, \cdot)$ bounded by $M$.
\end{theorem}

Conditions C1 and C2 in Theorem~\ref{thm:meta-generalization-maurer} define uniform stability (\ie, sensitivity of the algorithm to removal of an arbitrary point from the training sample~\citep{bousquet2002stability}) and state that the bound holds if the meta-algorithm $\Ab$ and every algorithm $A$ it produces are uniformly $\beta^\prime$- and $\beta$-stable with respect to the empirical risk $\hat R$ and a loss function $\ell$, respectively.
The bound becomes non-trivial when $\beta^\prime = o(1/n^a), a \geq 1/2$ and $\beta = o(1/m^b), b \geq 0$.

Theorem~\ref{thm:meta-generalization-maurer} provides a bound on the difference between the transfer risk $\Rc[\Ab(\Sb), \Pb(\Tc)]$ and its empirical estimator $\Lc_\mathrm{emp}(\Ab(\Sb); \Sb)$ based on meta-sample $\Sb$, implying that a small $\Lc_\mathrm{emp}(\Ab(\Sb); \Sb)$ guarantees meta-generalization within the bound.
Denoting $A \equiv \Ab(\Sb)$ to simplify our notation, the bound is obtained as follows:
\begin{align}
    \label{eq:maurer-emp-obj-bound}
	\Rc(A, \Pb(\Tc)) - \Lc_\mathrm{emp}(A; \Sb) =\ & \ep[\Dc \sim \Pb(\Tc)]{\ep[S \sim \Dc^m]{\hat R(A, S)}} - \frac{1}{n}\sum_{i=1}^n \hat R(A, S_i)\ + \\
	\label{eq:maurer-inner-task-bound}
	   & \ep[\Dc \sim \Pb(\Tc)]{\ep[S \sim \Dc^m]{R(A(S), \Dc) - \hat R(A, S)}}
\end{align}
The term \eqref{eq:maurer-emp-obj-bound} is the difference between the expected empirical risk over the true distribution of tasks and its estimate $\Lc_\mathrm{emp}(A; \Sb)$ based on the meta-sample $\Sb$.
As long as $\Ab$ is $\beta^\prime$-uniformly stable with respect to $\hat R(A, S)$ (C1, Theorem~\ref{thm:meta-generalization-maurer}), this term is bounded by $2 \beta^\prime + (4n\beta^\prime + M)\sqrt{\ln(1/\delta)/2n}$, which follows directly from the classical result of \citet{bousquet2002stability}.

The term \eqref{eq:maurer-inner-task-bound} is the estimation error of a model $f(\cdot) = A(S)$ learned by $A$ from $S$ with respect to the data distribution $\Dc$, computed in expectation over the distribution of tasks $\Pb(\Tc)$.
Stability of the inner-loop (C2, Theorem~\ref{thm:meta-generalization-maurer}) directly implies a bound of $2\beta$ on this term (see Theorem 6 in~\citep{maurer2005algorithmic}).
Putting together bounds of terms \eqref{eq:maurer-emp-obj-bound} and \eqref{eq:maurer-inner-task-bound}, we arrive at \eqref{eq:maurer-bound-app}.

\subsection{Bounding Meta-generalization of Reptile, MAML, and ProtoNets}
\label{app:meta-generalization-modern-methods}

The bound given in \eqref{eq:maurer-bound-app} is on the generalization error, \ie, the deviation of the true transfer risk $\Rc$ from the empirical estimator $\Lc_\mathrm{emp}$, and has meaningful practical implications only when the meta-algorithm $\Ab$ minimizes $\Lc_\mathrm{emp}$.
As we have shown in \ref{app:reptile-objective-function}, $\Lc_\mathrm{emp}(A; \Sb)$ is the meta-training objective function optimized by Reptile, and thus the bound from Theorem~\ref{thm:meta-generalization-maurer} applies directly.
However, MAML and ProtoNets optimize $\Lc_Q(A; \Sb)$, so we have to bound $\Rc(A, \Pb) - \Lc_Q(A; \Sb)$ instead, which can be decomposed into two terms similar to \eqref{eq:maurer-emp-obj-bound} and \eqref{eq:maurer-inner-task-bound}, where $\hat R$ is replaced by $\hat R_Q$ and $S$ is replaced by $S \setminus Q$ (since samples from the query set $Q$ are not used in the inner-loop).
The bound on the first term will not change much as we can still directly apply results from stability theory with the only caveat that we would require $\beta^\prime_Q$-uniform stability of the meta-algorithm with respect to $\hat R_Q$.
The second term, however, vanishes:
\begin{align}
	\MoveEqLeft \ep[S \sim \Dc^m]{R(A(S \setminus Q), \Dc) - \hat R_Q(A, S)} \nonumber \\
	=\ & \ep[S \setminus Q \sim \Dc^{m-k}]{R(A(S \setminus Q), \Dc) - \ep[Q \sim \Dc^{m-k}]{\frac{1}{|Q|} \sum_{(x, y) \in Q} \ell(A(S \setminus Q)(x), y)}} \equiv 0
\end{align}

This allows us to reformulate Theorem~\ref{thm:meta-generalization-maurer} and obtain the following generalization bound applicable to any meta-learning method that optimizes $\hat R_Q$ in the outer loop, including MAML and ProtoNets.

\begin{theorem}
    \label{thm:meta-generalization-bound-rq-app}
	Let the meta-algorithm $\Ab$ satisfy the following two conditions:
	\begin{enumerate}[itemsep=1pt,topsep=0pt,leftmargin=21pt]
		\item[C1.] For every pair of meta-samples $\Sb = \{S_1, \dots, S_n\}$, $\Sb^{-i} := \Sb \setminus \{S_i\}$, and for any sample $S$, we have $|\hat R_Q(\Ab(\Sb), S) - \hat R_Q(\Ab(\Sb^{-i}), S)| \leq \beta_Q^\prime$.
		\item[C2.] For any pair of samples $S = \{(x_1, y_1), \dots, (x_m, y_m)\}$, $S^{-j} := S \setminus \{(x_j, y_j)\}$, any algorithm $A$ produced by $\Ab$, and any $(x, y)$, we have $|\ell(A(S)(x), y) - \ell(A(S^{-j})(x), y)| \leq \beta$.
	\end{enumerate}
	Then for any task distribution $\Pb(\Tc)$, with probability at least $1 - \delta$ the following inequality holds:
	\begin{equation}
		\label{eq:maurer-bound-rq-app}
		\Rc(\Ab(\Sb), \Pb) - \Lc_Q(\Ab(\Sb); \Sb) \leq 2\beta_Q^\prime + (4n\beta_Q^\prime + M)\sqrt{\frac{\ln(1/\delta)}{2n}},
	\end{equation}
	where $\Lc_Q(\Ab(\Sb); \Sb) := \frac{1}{n} \sum_{i=1}^n \hat R_Q(\Ab(\Sb), S_i)$, $\hat R_Q(A, S_i) := \frac{1}{|Q_i|} \sum_{(x, y) \in Q_i} \ell(A(S_i \ Q_i)(x), y)$ with the loss function $\ell(\cdot, \cdot)$ bounded by $M$.
\end{theorem}

Since MAML, Reptile, and ProtoNets use stochastic gradient method (SGM) for solving the outer loop optimization problem, and Reptile additionally uses SGM in the inner loop as well, we further adopt the following general result from stability theory of SGM due to \citet{hardt2015train}.

\begin{lemma}[Theorem 3.12 in \citep{hardt2015train}]
	\label{thm:sgm-stability}
	Let $\ell(\cdot, z) \in [0, 1]$ be L-Lipschitz and $\gamma$-smooth loss function for every $z$.
	Suppose that we optimize $\frac{1}{n} \sum_{i=1}^n \ell(\theta, z_i)$ by running SGM for $T$ steps with monotonically non-increasing step sizes $\alpha_t \leq c/t$.
	Then, SGM is $\beta$-uniformly stable with
	\begin{equation}
		\beta \leq \frac{1 + 1/(\gamma c)}{n - 1} (2cL^2)^{1 / (\gamma c + 1)}T^{1 - 1/(\gamma c + 1)}
	\end{equation}
\end{lemma}
Combining Theorems~\ref{thm:meta-generalization-maurer}, \ref{thm:meta-generalization-bound-rq-app}, and \ref{thm:sgm-stability} we finally arrive at the meta-generalization error bounds for modern meta-learning algorithms.

\begin{theorem}
	\label{thm:meta-generalization-bounds-app}
	Let the meta-algorithm $\Ab$ be an SGM that optimizes an $L^\prime$-Lipschitz and $\gamma^\prime$-smooth loss $\Lc(A; \Sb)$ by taking $T^\prime$ steps with non-increasing step sizes $\alpha^\prime_t \leq c^\prime / t$.
	With probability at least $1 - \delta$, we have the following:
	\begin{enumerate}[itemsep=1pt,topsep=0pt,leftmargin=20pt]
		\item If $\Lc(A; \Sb)$ is Q-estimator of the transfer risk, then the following bound holds:
		\begin{equation}
			\label{eq:meta-generalization-bound-sgm-q-app}
			\Rc[A, \Pb(\Tc)] - \Lc(A; \Sb) \leq B^\prime(n, T^\prime, L^\prime, \gamma^\prime, c^\prime) \approx O\left(L^{\prime2}T^\prime\sqrt{\frac{\ln(1/\delta)}{n}}\right)
		\end{equation}
		\vspace{-2ex}
		\item If $\Lc(A; \Sb)$ is the empirical estimator of the transfer risk and the inner loop learning algorithm $A$ is an SGM that optimizes $L$-Lipschitz and $\gamma$-smooth loss $\ell(f(x), y)$ by taking $T$ steps with non-increasing step sizes $\alpha_t \leq c / t$, then:
		\begin{equation}
		    \label{eq:meta-generalization-bound-sgm-emp-app}
			\Rc[A, \Pb(\Tc)] - \Lc(A; \Sb) \leq B^\prime(n, T^\prime, L^\prime, \gamma^\prime, c^\prime) + B(m, T, L, \gamma, c) \approx O\left(L^{\prime 2}T^\prime\sqrt{\frac{\ln(1/\delta)}{n}} + L^2T\frac{1}{m}\right)
		\end{equation}
	\end{enumerate}
\end{theorem}
\begin{proof}
Conditions of the theorem and Lemma~\ref{thm:sgm-stability} imply that $\Ab$ is $\beta^\prime$-(or $\beta_Q^\prime$-)uniformly stable and the coefficient can be expressed through the Lipschitz and smoothness constants of $\Lc_\mathrm{emp}$ (or $\Lc_Q$).
This leads to the following expression for $B^\prime(n, T^\prime, L^\prime, \gamma^\prime, c^\prime)$:
\begin{equation}
    \label{eq:meta-generalization-bound-sgm-q-full}
    B^\prime(n, T^\prime, L^\prime, \gamma^\prime, c^\prime)
    = \frac{2C}{n}\left(1 + \frac{1}{n-1}\right) + 2C\sqrt{\frac{2\ln(1/\delta)}{n}}\left(1 + \frac{1}{n-1} + \frac{M}{4C}\right),
\end{equation}
where $C := (1 + 1/(\gamma^\prime c^\prime))(2 c^\prime L^{\prime2})^{1/(\gamma^\prime c^\prime + 1)} T^{\prime 1 - 1/(\gamma^\prime c^\prime + 1)}$.
The simplified expression given in \eqref{eq:meta-generalization-bound-sgm-q-app} upper-bounds \eqref{eq:meta-generalization-bound-sgm-q-full}.
Similarly, if each algorithm $A$ produced by the meta-algorithm $\Ab$ is an SGM on the $\Lc_\mathrm{emp}$ objective, using Lemma~\ref{thm:sgm-stability} we arrive at the following expression for $B(m, T, L, \gamma, c)$:
\begin{equation}
	B(m, T, L, \gamma, c) = 2\beta \leq 2\frac{1 + 1/(\gamma c)}{m - 1} (2cL^2)^{1 / (\gamma c + 1)}T^{1 - 1/(\gamma c + 1)} \approx O\left(L^2T\frac{1}{m}\right)
\end{equation}
where the approximation ignores terms associated with $c$ and $\gamma$.
The statement of the theorem now follows from Theorems~\ref{thm:meta-generalization-maurer} and ~\ref{thm:meta-generalization-bound-rq-app} and the derived expressions.
\end{proof}

Besides the implications of our theory discussed in the main text, we can make a few more interesting observations.

\paragraph{What happens if we use empirical estimator of the transfer risk as the objective for MAML?}
In principle, we can make MAML optimize $\Lc_\mathrm{emp}$ instead of $\Lc_Q$ in the outer loop.
\citet[Section~6.3,][]{nichol2018reptile} considered an interesting setup in their ablation study, where they analyzed how the overlap between the support and query data affects performance of the the first-order version of MAML.
Note that the larger the overlap, the closer MAML's objective becomes to $\Lc_\mathrm{emp}$.
Interestingly, they show that larger overlaps lead to the performance degradation on the Omniglot dataset.
This result is consistent with our theory---switching MAML's objective to $\Lc_\mathrm{emp}$ necessarily leads to larger meta-generalization error characterized by the additional $2\beta$ term in the bound.

\paragraph{Implications for federated learning.}
In federated learning research, one of the most popular algorithms is federated averaging (FedAvg)~\citep{mcmahan2016communication}, which uses model updates that are mathematically equivalent to Reptile.
The tasks are defined by the (private) datasets available on different client devices (\eg, mobile phones).
Our theory suggests that federated-averaging-style updates might be suboptimal for applications where the available labeled data for each client is very small; at the same time, when each client has sufficient data (as in the EMNIST dataset), we observe empirically superiority of Reptile/FedAvg over MAML (\cref{sec:exp-benchmarking}).
Designing personalized federated learning algorithms that learn by optimizing a combination of $\Lc_\mathrm{emp}$ (on clients with a lot of data) and $\Lc_Q$ (on clients with very small datasets) objectives is an interesting research avenue to explore next.

\section{Details on the Experimental Setup}
\label{app:exp-details}

We provide details on the experimental setup used throughout the paper, including model architectures (often termed \emph{backbone networks} in the few-shot learning literature) and hyperparameters for meta-learning methods.
Additionally, our full experimental configurations can be found in the provided supplementary code in the corresponding \texttt{conf/} folders, which enables full reproducibility.

\subsection{Network Architectures}
\label{app:exp-details-architectures}

For all our experiments, we used the standard Conv4 backbone network architectures proposed in the original papers~\citep{finn2017maml, snell2017prototypical, nichol2018reptile}.
The embeddings computed by the last hidden layer of the backbone networks were subsequently used for clustering in our active sampling approach.
MAML and Reptile used a linear final layer to compute logits from the embeddings, while ProtoNets used the distances between the query and support samples in the embedding space for computing class probabilities.

\textbf{\omniglot and EMNIST.}
Input images were resized to $28 \times 28$.
Models used by all methods consisted of 4 convolutional layers with 64 filters, kernel size of 3, and strides of 2, followed by batch normalization and ReLU activations (with no pooling or dropout in the intermediate layers).

\textbf{\miniimagenet.}
Input images were resized to $84 \times 84$.
Models used by all methods consisted of 4 convolutional layers with 32 filters, kernel size of 3, and strides of 2, followed by batch normalization and ReLU activations (with no pooling or dropout in the intermediate layers).

\subsection{Meta-learning Algorithms}
\label{app:exp-details-algorithms}

Meta-training (\ie, the outer loop optimization) was performed using Adam optimizer~\citep{kingma2014adam} with learning rate of 0.005 and $\beta_1 = 0$ for MAML and ProtoNets.
For Reptile, following parameters provided by \citet{nichol2018reptile}, the outer loop learning rate was set 1.0 and the optimizer set to stochastic gradient descent (SGD).
Details on model adaptation are provided below.

\textbf{MAML~\citep{finn2017maml}.}
At training time, we used 5 inner loop gradient descent (GD) steps with a learning rate of 0.01.
At evaluation time, the number of inner loop steps was set to 10.
To implement first order adaptation updates, we nullified the second order derivatives when computing the meta-training loss.

\textbf{Reptile~\citep{nichol2018reptile}.}
At training time, we used 10 inner loop gradient descent (GD) steps with a learning rate of 0.001 for \omniglot and 0.0005 for \miniimagenet.
At evaluation time, the number of inner loop steps was set to 50.

\textbf{Prototypical Networks~\citep{snell2017prototypical}.}
We used a version of the method with the Euclidean distance.
The method has no other hyperparameters besides those of the outer loop optimizer.

\subsection{Calibration}
\label{app:exp-details-calibration}

\begin{table}[t]
    \centering
    \caption{Meta-test performance with unbounded supervision.}
    \label{tab:calibration}
    \small
    \renewcommand{\arraystretch}{1.2}
    \begin{tabular}{l|rr|rr|rr}
        \toprule
        \textbf{Method} & \textbf{O-5w-1s}  & \textbf{O-5w-5s}  & \textbf{O-20w-1s} & \textbf{O-20w-5s} & \textbf{MI-5w1d} & \textbf{MI-5w5d} \\
        \midrule
        MAML            & $98.3 \std{0.6}$  & $99.9 \std{0.1}$  & $95.0 \std{0.5}$  & $98.6 \std{0.5}$ & $48.7 \std{1.7}$ & $63.0 \std{0.9}$ \\
        Reptile         & $94.9 \std{0.2}$  & $98.2 \std{0.5}$  & $88.2 \std{0.4}$  & $96.4 \std{0.4}$ & $47.8 \std{1.3}$ & $61.9 \std{1.1}$ \\
        Protonets       & $97.9 \std{0.4}$  & $99.0 \std{0.1}$  & $91.9 \std{1.2}$  & $98.6 \std{0.5}$ & $48.3 \std{0.8}$ & $66.2 \std{0.8}$ \\
        \bottomrule
    \end{tabular}
\end{table}

We selected the hyperparameters described above such that the meta-test performance of all methods nearly matched the reported numbers in the original papers in the limited supervision regime.
Results for the calibrated models are reported in Table~\ref{tab:calibration}.

\textbf{A note on Reptile.}
\citet{nichol2018reptile} used 10-shot tasks at meta-training time and trained for over 100,000 meta-updates (each meta-update was computed on a batch of 20 tasks) in order to attain the performance reported in the original paper.
In the limited supervision setting, this would have required a label budget of over 100M (\ie, 1000 times larger than those considered in our study).
However, just for calibration purposes, we matched the original setup of \citet{nichol2018reptile}.

\textbf{A note on ProtoNets.}
To improve performance, \citet{snell2017prototypical} proposed to meta-train ProtoNets on tasks with higher number of classes than the tasks used at meta-test time (\eg, meta-training on 60-way tasks while meta-testing on 20-way tasks).
Even though training tasks with more classes could be helpful in learning better data representations, increasing the number of classes per task affects the amount of labeled points required per task and may affect performance of non-oracle label selection strategies.
Therefore, in our experiments, we decided to stick with a clean setup that matches the number of classes per task at both meta-training and meta-test times, although sacrificing some performance gains.
Again, for calibration purposes only, we used an increased number of classes per task at meta-training time.

\subsection{Limitations}
\label{app:exp-details-limitations}

To avoid a combinatorially large number of combinations of architectures, algorithms, and their hyperparameters, we had to fix many of these variables before experimenting with different labeling budgets and sampling strategies.
While this allowed us to conduct a fairly comprehensive study of 3 different meta-learning methods across a variety of regimes, the reported results may be limited to the specific choice of the setup described above; we do not exclude the possibility that the behavior of different methods might vary with the setup (\eg, tuning hyperparameters for each labeling budget separately, while extremely costly, might have rectified poor performance of some of the methods on some of the benchmarks).

\end{document}